\documentclass[10pt,twocolumn,letterpaper]{article}

\usepackage{cvpr}
\usepackage{times}
\usepackage{epsfig}
\usepackage{graphicx}
\usepackage{amsmath}
\usepackage{amssymb}
\usepackage{amsfonts} 
\usepackage{amsthm}
\usepackage{colonequals}
\usepackage{algorithm}
\usepackage{algorithmic}
\usepackage{booktabs}
\usepackage{float}

\newtheorem{theorem}{Theorem}[section]

\theoremstyle{definition}
\newtheorem{definition}[theorem]{Definition}

\newtheorem{claim}[theorem]{Claim}

\usepackage[pagebackref=true,breaklinks=true,letterpaper=true,colorlinks,bookmarks=false]{hyperref}

\cvprfinalcopy 

\def\cvprPaperID{1823} 

\ifcvprfinal\fi
\begin{document}

\newcommand{\bigo}{\mathcal{O}}
\newcommand{\bigom}{\mathcal{O}^{\mathrm{MEM}}}
\newcommand{\bigof}{\mathcal{O}^{\mathrm{MAC}}}
\newcommand{\mrk}[1]{%
{\color{red}{#1}}
}

\title{WaveletNet: Logarithmic Scale Efficient Convolutional Neural Networks for Edge Devices}

\author{
Li Jing${^*}$, Rumen Dangovski${^*}$, Marin Solja\v{c}i\'{c} \\
  Massachusetts Institute of Technology\\
Cambridge, MA \\
\texttt{\{ljing, rumenrd, soljacic\}@mit.edu} \\
${^*}$ indicates equal contribution
}

\maketitle

\begin{abstract}
   We present a logarithmic-scale efficient convolutional neural network architecture for edge devices, named WaveletNet. Our model is based on the well-known depthwise convolution, and on two new layers, which we introduce in this work: a wavelet convolution and a depthwise fast wavelet transform. By breaking the symmetry in channel dimensions and applying a fast algorithm, WaveletNet shrinks the complexity of convolutional blocks by an $\mathcal{O}(\log{D}/D)$ factor, where $D$ is the number of channels.  Experiments on CIFAR-10 and ImageNet classification show superior and comparable performances of WaveletNet compared to state-of-the-art models such as MobileNetV2.
\end{abstract}

\section{Introduction}
Convolutional neural networks (CNNs) have dominated computer vision since deep CNNs, such as AlexNet~\cite{krizhevsky2012imagenet}, began winning the ImageNet Challenge: ILSVRC 2012~\cite{ILSVRC15}. To achieve high accuracy, researchers have been increasing the depth and complexity of CNNs to extreme boundaries of hardware limits (\cite{43022, Simonyan14c, He2016DeepRL, Szegedy2016RethinkingTI, Huang2017DenselyCC, Szegedy2017Inceptionv4IA}). As a result, more and more computational power is required to run such complicated networks. Because of their limited computational power, as well as memory space, many real world applications, such as robotics, self-driving cars, augmented reality and smartphones are unable to afford such CNNs. Hence, the necessity and efficiency of so many layers and computationally intensive architectures is being challenged. 

Recently, researchers have been focusing on developing efficient, mobile-friendly neural networks such as MobileNet~\cite{Howard2017MobileNetsEC}, ShuffleNet~\cite{Zhang2017ShuffleNetAE} and their improved versions (V2s)~\cite{Sandler2018MobileNetV2IR, Ma2018ShuffleNetVP}. These models significantly shrink the conventional CNNs without decreasing accuracy noticeably. The key to such reductions is lowering the number of operation between channel dimensions through \textit{group convolutions} (GConv) \cite{krizhevsky2012imagenet} and \textit{depthwise convolutions} (DWConv) \cite{Chollet2017XceptionDL}. In terms of complexity on the channel dimension, these convolutions utilize two extremes, $\bigo(D^2/g)$ and $\bigo(D)$ scales respectively, where $g$ is the number of groups in a GConv. 

\begin{figure}[!t]
	\centering
  \includegraphics[width=0.9\linewidth]{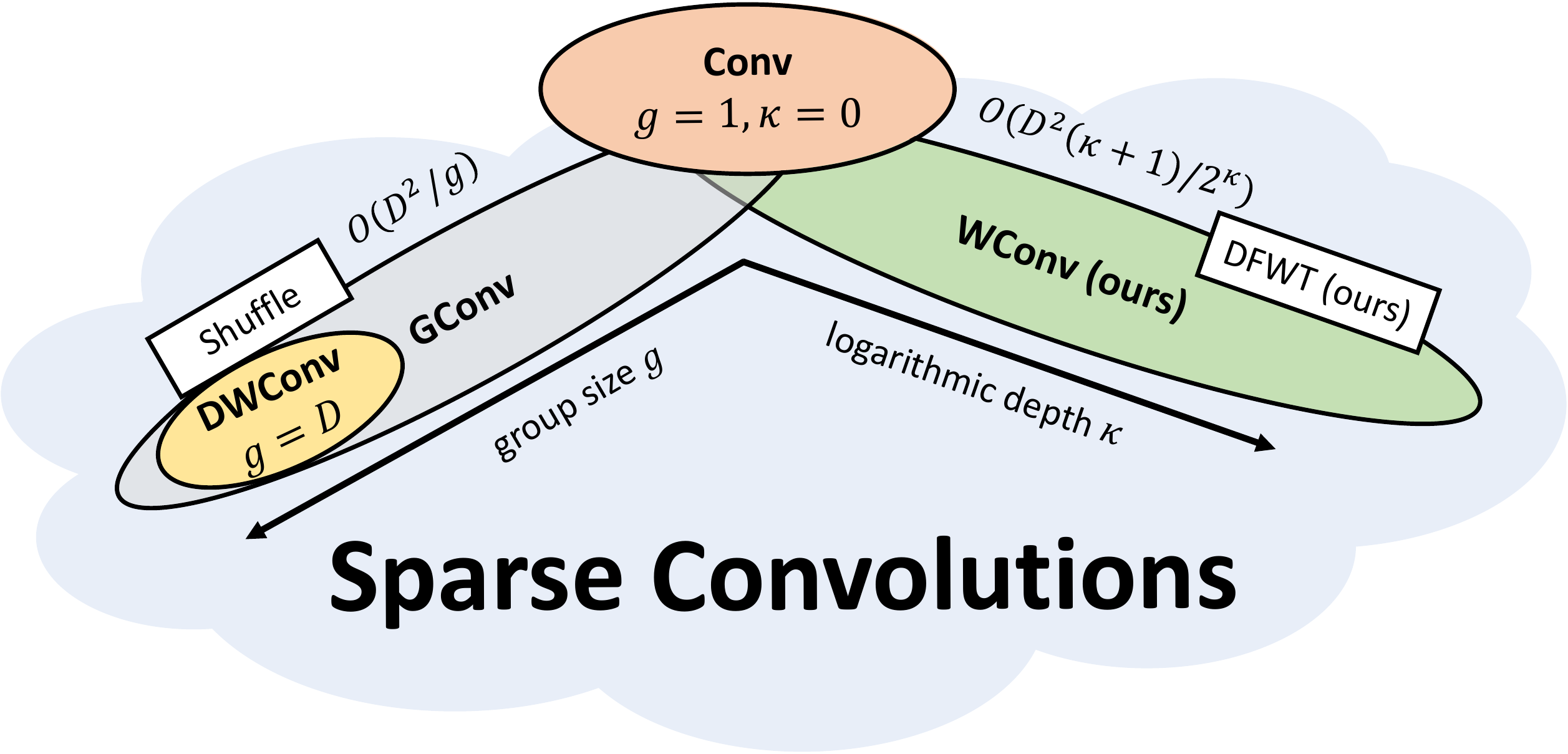}
  \caption{Space of sparse convolutions. Our contribution is a new direction of efficient convolutions (in green). Input and output channels are both assumed to equal $D$. Best viewed in color. \label{fig:intro}}
  \vskip -0.2in
\end{figure}

Automatic machine learning (AutoML) uses these human designed architectures as an inductive bias for automatic architecture search, which provides state-of-the-art efficiency~\cite{Zoph2017LearningTA, Liu2018PNAS, Pham2018ENAS, Tan2018MNAS}. Inspired by models obtained from AutoML we identify that \textit{breaking the symmetry} is important for efficient models. 
Intuitively, keeping a symmetric structure for each building block of the CNN could generate redundant features and thus the limited size of edge models might lead to low accuracy. We propose to break the symmetry in the \textit{channel} dimensions instead, thereby introducing channel imbalance for log-scale computations. To our knowledge, channel imbalance has not been previously explored by AutoML nor by human designs, thus we choose this new architecture space to aim at efficiency.

Our contributions' fit in research on sparse convolutions is shown in Figure~\ref{fig:intro}. We find that a key for sparsity/accuracy trade-off is to maintain connectivity between channel dimensions within residual blocks; in combination with channel imbalance this provides efficient flow of information. Therefore, we propose a new convolutional layer, named \textit{wavelet convolution} (WConv) and its conjugate layer to keep connectivity: \textit{depthwise fast wavelet transform} layer (DFWT). WConv breaks the channel symmetry and reduces complexity to $\log D$ while DFWT takes negligible operations. DFWT itself is derived from an optimization through a formalism on connectivity that we develop to utilize the space of imbalanced channel dimensions.  Hence, our building block for CNNs yields full connectivity and attains the efficiency/accuracy trade-off of MobileNetV2 but requires log-scale complexity: an unmatched feat by the state-of-the-art.

Using our building blocks, WConv and DFWT, we construct a new efficient CNN called \textit{WaveletNet}. We compare WaveletNet to other extremely small models such as MobileNetV1/2 and ShuffleNet, on CIFAR-10 and ImageNet. We find that WaveletNet outperforms state-of-the-art models on CIFAR-10. Furthermore, we conduct an ablation study to signify the importance of the DFWT layer. We find that DFWT is important for boosting the accuracy of WConv, just like the Shuffle layer improves grouped convolution in ShuffleNet~\cite{Zhang2017ShuffleNetAE}.  Our experiments on ImageNet achieve 30.6\% error with only 216M MACs (multiply-accumulate operations) and 2.72M parameters thereby showing that the accuracy of WaveletNet is comparable to that of MobileNetV2.

Our contributions are as follows: we (\emph{i}) inspired by AutoML first explore breaking the symmetry in channel dimensions, which may open a new path for human designed architecture search; (\emph{ii}) provide connectivity guidelines for development of efficient CNNs; (\emph{iii}) develop two new convolutional layers: WConv and DFWT that serve as building blocks for any CNN; (\emph{iv}) provide a state-of-the-art efficient CNN for edge devices.

\section{Related work} \label{sec:rel-work}

\subsection{Efficient models approach}

The computer vision community has been well-focused on building efficient convolutional neural networks that can match modern day hardware. For example, ever since AlexNet \cite{krizhevsky2012imagenet}, group convolution has been widely used to decrease the model complexity, thereby addressing limited GPU memory and achieving parallel GPU processing. Methods, such as factorization \cite{Wang2017FactorizedCN} and depthwise separable convolutions \cite{Chollet2017XceptionDL}, have demonstrated successful, lighter, structures that have significantly smaller numbers of parameters, compared to conventional 2D convolutional layers (Conv). Extremely small models, such as MobileNet \cite{Howard2017MobileNetsEC} and ShuffleNet \cite{Zhang2017ShuffleNetAE}, provide solutions that can achieve acceptable accuracy on ImageNet within small amount of computation budget. Other notable efficient models include \cite{2016Squeeze, 2016MergeAndRun, Zoph2017LearningTA, Huang2017DenselyCC, 2017IntGroup}. 
Our work falls into this category.

\subsection{Automatic architecture search}
Outside the scope of these human designed architectures, recent work utilizes reinforcement learning and model searching to construct highly efficient networks through the NasNet framework~\cite{Zoph2017LearningTA}. NASNet~\cite{Zoph2017LearningTA} first demonstrated that AutoML is able to design more efficient architectures than humans. A shortage of such methods is that they require huge amount of GPU hours which is not affordable by the average researcher. Therefore, PNASNet~\cite{Liu2017ProgressiveNA} and ENASNet~\cite{Pham2018EfficientNA} aimed at improving the efficiency of the architecture search to narrow down to hundreds of GPU hours. Moreover, a promising trend is gaining insights from these meta-learning result and extracting guidelines for designing efficient models~\cite{Tan2018MNAS}.

\section{Preliminaries, discussion and intuition}

\subsection{Full connectivity in a residual block}
Since AlexNet~\cite{krizhevsky2012imagenet} grouped convolution has been widely used in convolutional neural network architectures to bring sparsity and save memory requirement. Grouped convolution divides the tensor along the channel dimensions equally into several isolated parts and applies convolutions on each part. As a result, both number of parameters and operation counts shrink by a factor of $g$ where $g$ is the number of groups. However, grouped convolution introduces a trade--off: the sparsity caused by grouped convolution leads to decrease in accuracy. To combat this the dimensions need to be tuned carefully to make the sparsity/accuracy trade-off worth because, alternatively, one can always just reduce the number of channels. Several works find that special designs of grouped convolution or extra depthwise operations can compensate the trade--off and provide profitable sparsity. Examples of such models include ShuffleNet~\cite{Zhang2017ShuffleNetAE}, IGCV~\cite{IGCV1} and their variations~\cite{IGCV2, IGCV3, Ma2018ShuffleNetVP}.

We observe that connectivity is key for maintaining accuracy when using grouped convolutions. Within a residual block we say that an input channel is connected to an output channel if there is a computational path between the two, i.e. if the output channel depends on the information in the input channel. By \textit{full} connectivity we mean that every input channel is connected to every output channel. In the Supplementary Material we build a connectivity formalism to study this.

Our intuition is as follows: a convolutional block generates a higher level representation of its input tensor. Each input feature is encoded as a superposition of the input channels. An output feature is formed by combining input features and is itself encoded as a superposition of the output channels. Hence, it is important to connect each input channel to each output channel. Figure~\ref{fig:concept} (\emph{a}) sketches full connectivity where the new representation depends on all input features. In contrast, Figure~\ref{fig:concept} (\emph{b}) shows a potential failure of grouped convolution where limited connectivity is present.

\begin{figure}[t!]
	\centering
  \includegraphics[width=0.9\linewidth]{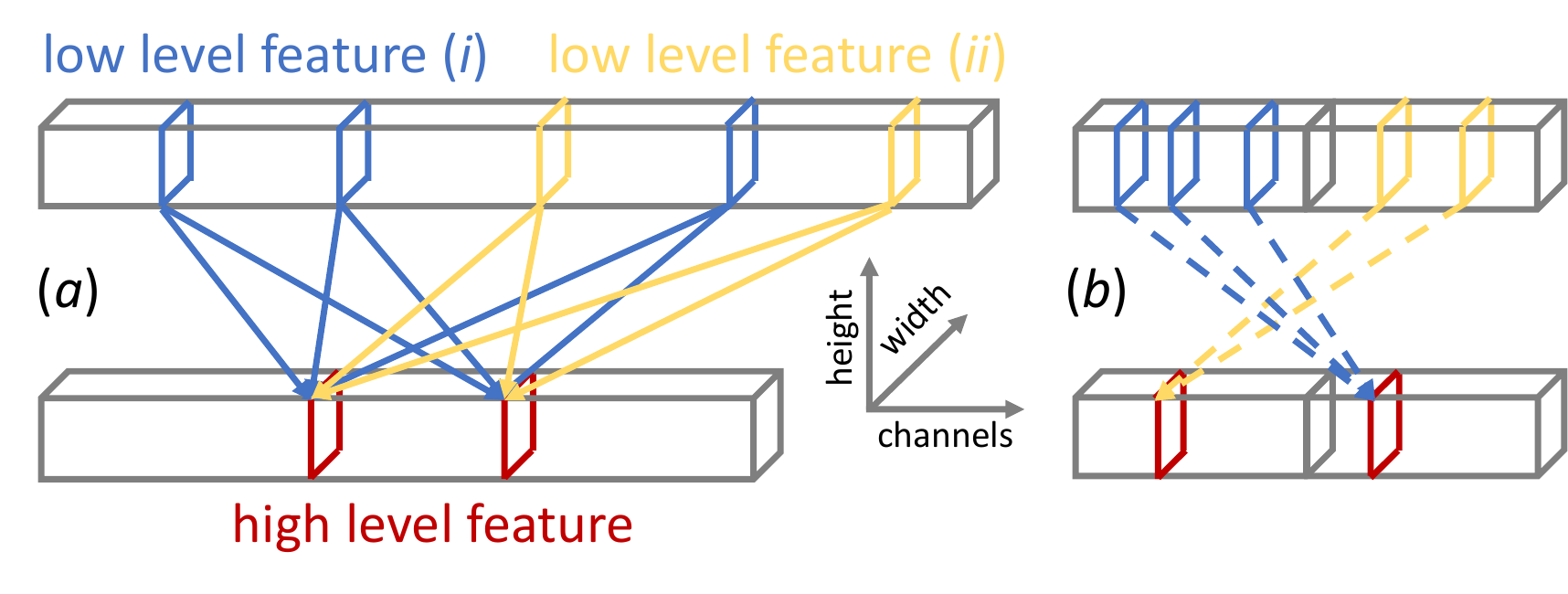}
  \caption{Full vs. limited connectivity. (\emph{a}) The high level feature is a superposition of channels (red) that are all connected with the low level features (\emph{i,ii}). (\emph{b}) grouped convolution (with two independent groups) limits connectivity, so the high level feature cannot connect with all channels in the blue and yellow superpositions (features (\emph{i,ii})).  Best viewed in color. \label{fig:concept}}
\end{figure}

Therefore, to keep full connectivity within each block, an extra depthwise layer is necessary. For example, the Shuffle layer provides connectivity to grouped convolutions, which boosted ShuffleNet's performance to achieve state-of-the-art efficiency/accuracy trade-off~\cite{Zhang2017ShuffleNetAE}.

\subsection{Asymmetry in grouped convolution}
AutoML provides higher accuracy in most classification tasks such as CIFAR-10 and ImageNet. Searching for an optimal model often comes at the expense of complicated meta-information describing the model. Thus, it is difficult for human researchers to understand the fundamental advantages of the architectures. Some attempts at reading AutoML's designs are being made, e.g. explaining the presence of more and more 5x5 kernels in MNasNet~\cite{Tan2018MNAS}, but it is still not very clear how such insights could benefit human design of efficient architectures. 

We believe such complications arise because AutoML is performed in a big search space, the space of \textit{asymmetric} models, where humans are not good at producing efficient designs. Indeed, human developed models are usually extremely symmetric in some degree. For example VGG, ResNet, MobileNet, ShuffleNet and others have strong symmetry in channel dimensions. 

There are several attempts to break the symmetry in convolutional neural networks. Inception models break the symmetry in spatial dimensions by applying more diversified window sized convolutions and provided state-of-the-art performance compared to prior models. Recently, SENet achieved the highest record on ImageNet because it provided a bigger design space outside the symmetric scope of channel dimensions~\cite{Hu2017SqueezeandExcitationN}.

Here, we propose breaking the symmetry in channel dimensions. Even this narrowed search space is too big to be explored, thus we design certain patterns to follow a specific functional path in this space. We model our design in a recursive way: a grouped convolution with recursively decreased depth.

\subsection{Advantage of fast algorithms}
Human designed architectures always choose a certain scope and hence restrict the search for models in a subspace of the overall functional space. Thus specially designed layers can be used in combination with efficient convolutions so that the reduced number of parameters does not decrease accuracy significantly. For example, the non-trainable Shuffle layer can be usefully attached to GConv (i.e. Shuffle acts as conjugate to GConv). Such conjugate layers can boost the performance of efficient CNNs significantly.   

It is important to be able to perform conjugate layers efficiently. With its simple description the Shuffle layer can be implemented very fast. Similarly, researches have tried use Fast Fourer Transforms in spatial dimensions~\cite{Rippel2015SpectralRF}.  

Here, we focus on applying fast algorithms in channel dimensions. Similar to the Shuffle layer, we propose a depthwise fast wavelet transformation (DFWT) which provides strong connectivity with minimal operations. DFWT establishes full connectivity in each block and significantly reduces the requirement for trainable parameters and overall complexity. As a result, this fast algorithm provides same effect but reduces the number of operations significantly.

\begin{figure}[!t]
	\centering
  \includegraphics[width=0.9\linewidth]{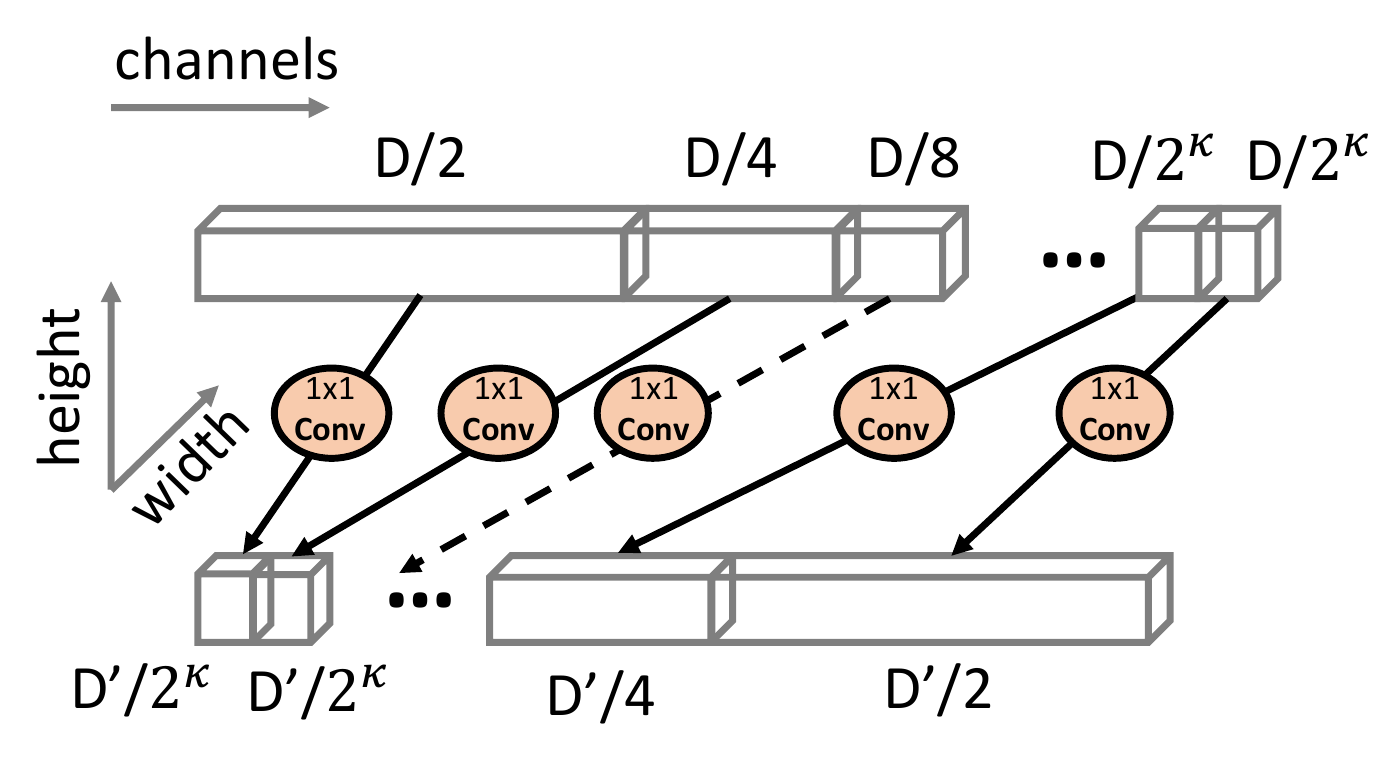}
  \caption{Architecture for the Wavelet Convolution WConv($\kappa$) with $D$ input, $D'$ output channels and depth $\kappa$. We always assume that $D$ and $D'$ are divisible by $2^\kappa$. The output piece for the dashed line is not shown. Note that $D$ may not equal $D'$. \label{fig:wconv}}
\end{figure}

\section{The WaveletNet Architecture}

\subsection{Wavelet Convolution (WConv)}
WConv is an asymmetric grouped convolution. Instead of dividing the input channel dimension $D$ into groups of equal size we propose an asymmetric partition parameterized by a logarithmic depth $\kappa$ as follows
\begin{equation} \label{eq:partition_in}
\{D/2,D/4,D/8,\dots,D/2^{\kappa - 1}, D/2^\kappa, D/2^\kappa \}
\end{equation}
We assume throughout the paper that $2^\kappa$ divides all channel dimensions, so that the sum of the integer pieces in Equation~\eqref{eq:partition_in} is $D/2^{\kappa}(1+1+2+4+\dots+2^{\kappa-1}) = D$, i.e. it adds up to $D$. To leverage efficiency, we match long input channels with short input channels by partitioning the output channel dimension $D'$ in a reversed way as follows 
\begin{equation} \label{eq:partition_out}
\{D'/2^\kappa,D'/2^\kappa,\dots,\dots, D'/2^4, D'/2 \}.
\end{equation}
We align partitions \eqref{eq:partition_in} and \eqref{eq:partition_out} and perform standard KxK convolutions (KxKConv) between corresponding pieces. Figure~\ref{fig:wconv} sketches WConv and Algorithm~\ref{alg:WConv} describes the formal procedure of implementing this layer. We call this procedure wavelet convolution because our partition resembles constructions with dyadic intervals, which have applications in wavelet analysis.

\begin{algorithm}[H]
  \caption{Wavelet Convolution (WConv)}
  \label{alg:WConv}
\begin{algorithmic}
  \STATE {\bfseries Input:} $\mathbf{x} \text{ (size } N\times N \times D), \kappa, D'$ //input (filer size x filter size x input channels), depth, output channels
  \STATE {\bfseries Output:} $\mathbf{y} \text{ (size } N\times N \times D')$
  \STATE $s\leftarrow 0$; $e\leftarrow D/2$ //start and end indices
  \STATE $\mathbf{y}_0\leftarrow$Conv($\mathbf{x}$[:,:,$s$:$e$],stride=1,out\_ch.=$D'/2^\kappa$)
  \FOR{$i=1$ {\bfseries to} $\kappa - 1$}
  \STATE $s\leftarrow e$
  \STATE $e\leftarrow e + D/2^{i+1}$
  \STATE $\mathbf{y}_i\leftarrow$Conv($\mathbf{x}$[:,:,$s$:$e$],stride=1,out\_ch.=$D'/2^{\kappa-i+1}$)
  \ENDFOR
  \STATE $\mathbf{y}_k\leftarrow$Conv($\mathbf{x}$[:,:,$D-D/2^\kappa$:$D$],stride=1,out\_ch.=$D'/2$)
  \STATE $\mathbf{y} \leftarrow \text{Concat}([\mathbf{y}_{0},\dots,\mathbf{y}_{k}])$ //on channel dimension
\end{algorithmic}
\end{algorithm}

\subsection{Depthwise Fast Wavelet Transform (DFWT)}
Similar to the Shuffle layer for grouped convolutions, a specific \textit{non-trainable} layer is the key for efficient utilization of WConv. We propose a new layer: the \textit{Depthwise Fast Wavelet Transform} (DFWT). Figure \ref{fig:dfwt} demonstrates our construction. 

DFWT effectively broadcasts a Haar Matrix on the channel dimension. This can be viewed as a convolution with no additional parameters to the model that only manipulates the channel dimension. Moreover, since Haar matrices are sparse, instead of an $\bigo(D^2)$ vector-matrix multiplication we can implement an extremely efficient $\bigo(D)$ multiplication through a relabeling tweak to the Fast (Haar) Wavelet Transform~\cite{Keiser1998FastHaar, WaveletBook2012}. Algorithm \ref{alg:DFWT} is our own simple and efficient implementation of what-we-call the \textit{modified} Fast Wavelet Transform, adapted to the DFWT layer. 

\begin{algorithm}
  \caption{Depthwise Fast Wavelet Transform (DFWT)\label{alg:DFWT}}
\begin{algorithmic}
  \STATE {\bfseries Input:} $\mathbf{x} \text{ (size } N\times N\times N)$, $\kappa$ //input (filer size x filter size x input channels), depth
  \STATE {\bfseries Output:} $\mathbf{y} \text{ (size } N\times N\times D)$
  \STATE $\mathbf{T}\leftarrow \mathbf{x}$ 
  \FOR{$i=1$ {\bfseries to} $\kappa$}
  \STATE $ n\leftarrow$ channel dimension of $\mathbf{T}$
  \STATE $\mathbf{T} \leftarrow \text{reshape}(\mathbf{T}, [n/2, 2]) \text{ //on channel dimension}$
  \STATE $\mathbf{y}_{i}\leftarrow\mathbf{T}[1] - \mathbf{T}[0]$ //on new channel dimension $n/2$
   
  \STATE $\mathbf{T} \leftarrow \mathbf{T}[1] + \mathbf{T}[0]$ //on new channel dimension $n/2$
  \ENDFOR
\STATE $\mathbf{y}_{\kappa} \leftarrow \mathbf{T}$ //remaining tensor from the end of the loop
\STATE $\mathbf{y} \leftarrow 
\text{Concat}([\mathbf{y}_1, \dots, \mathbf{y}_\kappa])$ //on channel dimension
\end{algorithmic}
\end{algorithm}

\begin{figure}[!t]
	\centering
  \includegraphics[width=0.9\linewidth]{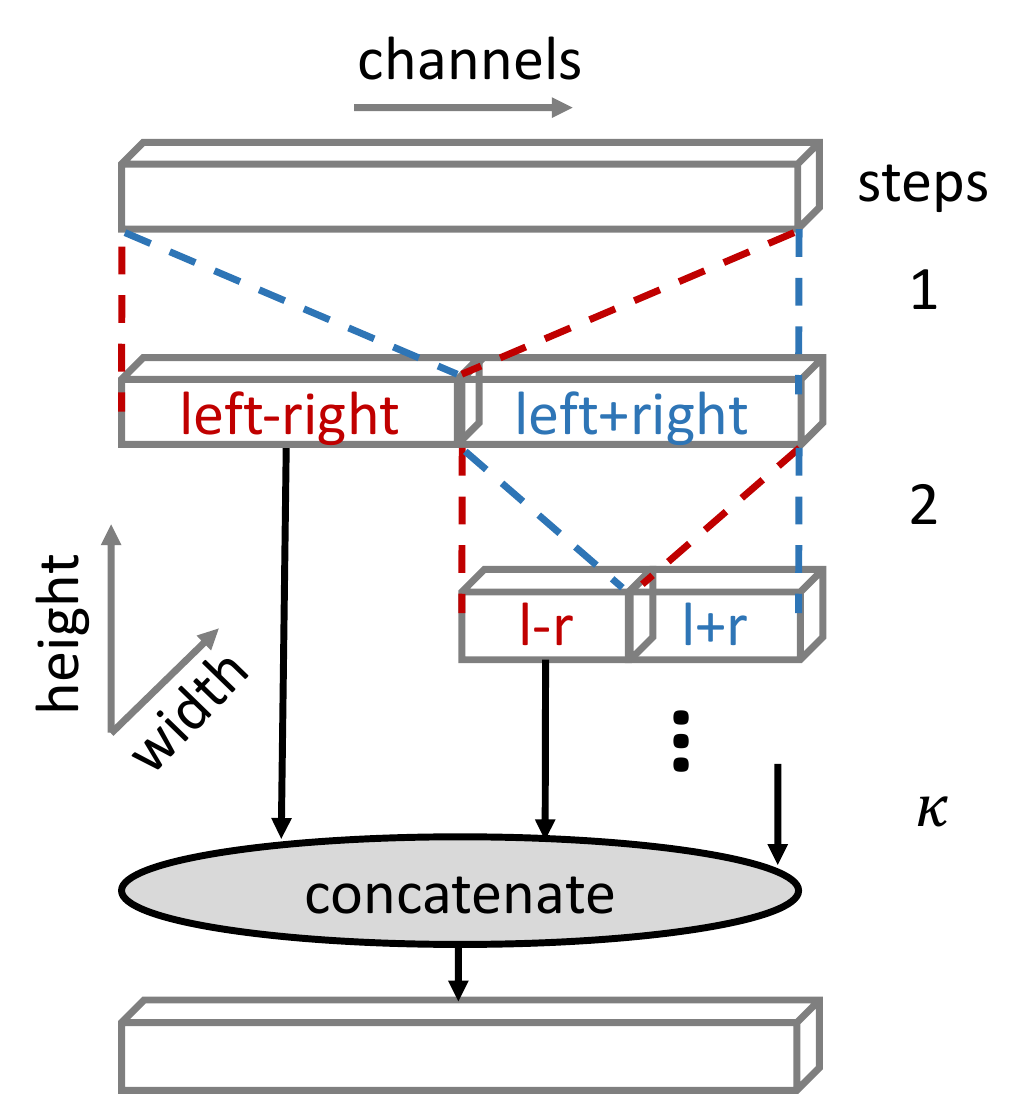}
  \caption{Architecture for the Depthwise Fast Wavelet Transform DFWT($\kappa$) with depth $\kappa$. For $\kappa$ steps the tensor features are split in half along the channels dimension. The difference of the two parts is concatenated at the output, while the sum is fed recursively to the next step. \label{fig:dfwt}}
\end{figure}

In the Supplementary Material, we use our formalism to prove that DFWT ensures full connectivity of the WConv layer. We demonstrate the importance of this property via an ablation study (see Table \ref{tabl:abl-study} for example). Assuming full connectivity we can derive the Haar matrix as a broadcasting element in DFWT. We call the new layer depthwise fast wavelet transform because of the applications of the Haar transform in wavelet analysis.

\subsection{The WaveletNet unit}
We propose \textit{WaveletNet}, which is a CNN, composed of unit structures of WConv, DFWT and DWConv.

The \textit{bottleneck block} has been proven useful ever since ResNet won ILSVRC in 2016~\cite{He2016DeepRL}. We develop our \textit{WaveletNet unit} based on this bottleneck block. The unit breakdown is listed in Figure~\ref{fig:wnunit}. We use a reverse residual block with expanstion ratio $t$ in a similar fashion to~\cite{Sandler2018MobileNetV2IR}. The three logarithmic depths through which we parameterize the first WConv, DFWT and the last WConv are respectively $\kappa_1$, $\kappa_2$ and $\kappa_3$. We conduct a grid search over small values for the logarithmic depths on CIFAR-10 and identify that the optimal parameters are $\kappa_1=0$, $\kappa_2=3$ and $\kappa_3=3$. We keep these values fixed throughout the paper.

We now study the tunability of our WaveletNet unit. From Figure~\ref{fig:intro} we note that when $\kappa_1=0$ then 1x1WConv($\kappa_1$) coincides with 1x1Conv. Analogously, setting $\kappa_3=0$ also gives 1x1Conv for 1x1WConv($\kappa_3$). This is in line with the fact that GConv and WConv converge to a standard convolution when we set the group size to be one and the logarithmic depth to be zero. Moreover, if we set $\kappa_2=0$ then the for loop in Algorithm~\ref{alg:DFWT} will not be executed, thus DFWT($\kappa_2$) becomes the identity function. Hence, combining all considerations and setting $\kappa_1,\kappa_2,\kappa_3$ to zeros we recover the block unit for MobileNetV2. Therefore MobileNetV2 appears as a corner case for our architecture. 

Most notably, however, the WaveletNet unit is flexible and more efficient outside of the above-mentioned corner case. In our model we aim to set the depth $\kappa$ large in order to approach log-scale complexity. 

\begin{figure}[!t]
	\centering
  \includegraphics[width=0.9\linewidth]{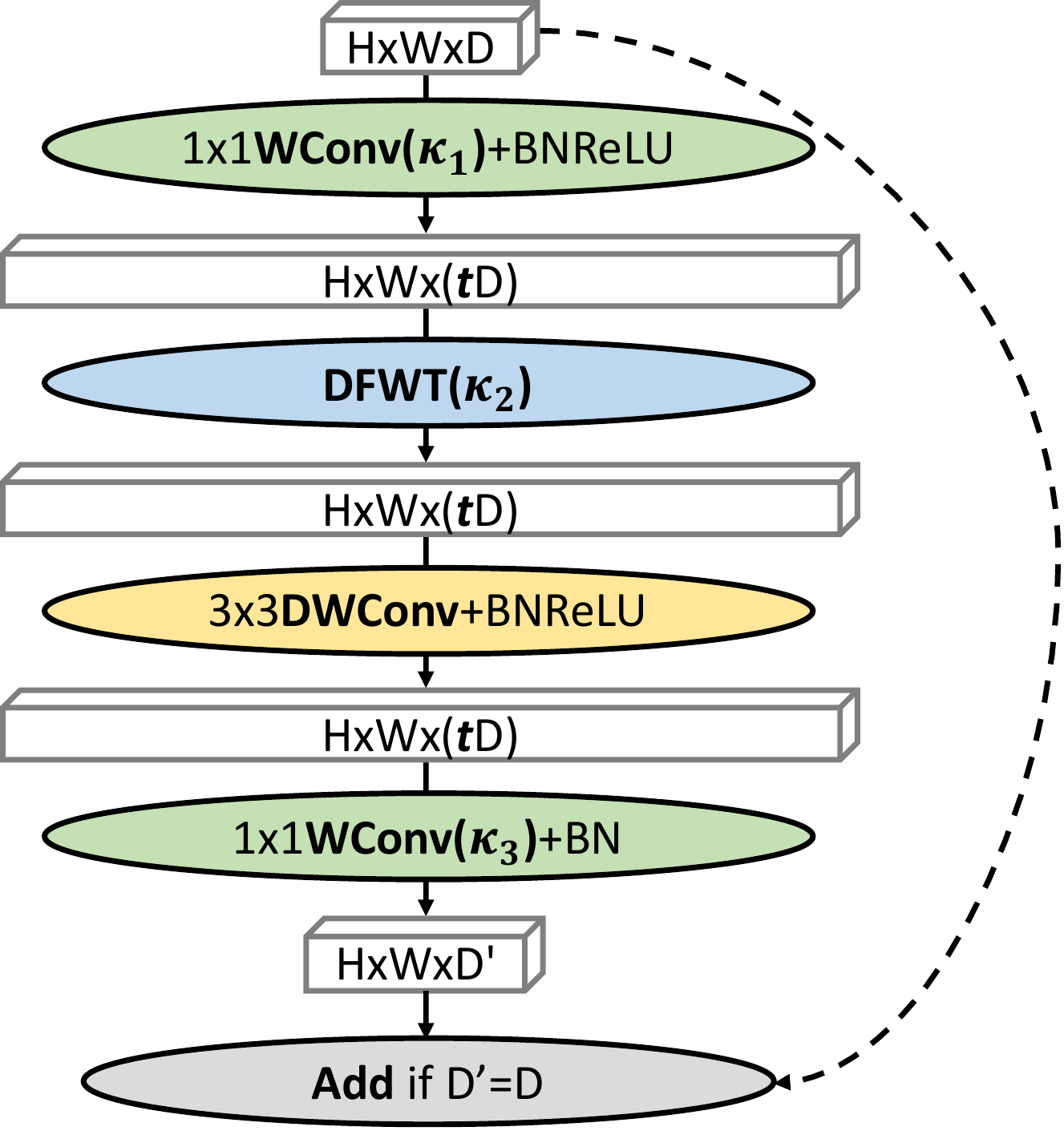}
  \caption{Architecture for the WaveletNet Unit with an expansion ratio $t$, depths $\kappa_1$, $\kappa_2$ and $\kappa_3$, input channels $D$ and output channels $D'$. \label{fig:wnunit}}
\end{figure}

\subsection{Complexity study}
Here, we analyze the model complexity. Assume the convolutional block has input dimension $N \times N \times D$ and output dimension $N \times N \times D'$. The operation and number of parameters for each layer is listed in the following table.
\begin{table}[h!]
    \small
    \centering
    \begin{tabular}{l|l|l}
        \hline
        conv layer & \#MAC & \#params \\
        \hline
        3x3DWConv  & $9N^2D'$  & $9D$   \\
        1x1Conv          & $N^2DD'$  & $DD'$   \\
        1x1GConv (group $g$)      & $N^2DD'/g$  & $DD'/g$   \\
        \hline
        1x1WConv  & $N^2DD'(\kappa+1)/2^{\kappa}$  & $DD'(\kappa+1)/2^{\kappa}$   \\
        1x1WConv (max $\kappa$)    & $N^2D\log{D'}$  & $D\log{D'}$   \\
        \hline
    \end{tabular}
    \caption{Complexity of WConv vs. other convolutional layers}
    \label{tab:my_label}
\end{table}

Usually, grouped convolutions use group size $g=3$~\cite{Zhang2017ShuffleNetAE}. Therefore, WConv will have less MAC and number of parameters if $\kappa>2$. This is mostly in line with our grid search because we find that the best trade-off for WaveletNet is achieved for $\kappa_2=\kappa_3=3$. The only deviation is $\kappa_1=0$, which leads to 1x1Conv as the first layer of the convolutional block. Empirically, we find keep $\kappa_1=0$ is responsible for maintaining good accuracy despite not approaching a log-scale complexity.

\subsection{A limitation}
The limitation of the current WaveletNet block unit is that we still effectively use a 1x1Conv layer as the first layer (since $\kappa_1=0$). As a result, we only replace one convolutional layer by a truly logarithmic WConv (e.g. $\kappa_3=3$). Therefore, the overall complexity is bounded by the 1x1Conv layer, which does not lead to a dominating logarithmic scale. Nevertheless, our results in the following section show that only by replacing one of the standard convolutional layers is enough to achieve a worthy efficiency/accuracy trade-off. This encourages us to explore equipping \textit{all} given layers with log-scale complexity for future work. Namely, we believe that further tuning of $\kappa_1$, $\kappa_2$ and $\kappa_3$ and exploration of hyperparameters could imporve our search for efficient log-scale models.

\section{Experiments} \label{sec:exp}
In this section, we present experimental results on Image Classification for different-sized WaveletNet models along with other state-of-the-art efficient CNNs. We use the datasets CIFAR-10 \cite{CIFAR10} and ImageNet  (ILSVRC 2012 \cite{ILSVRC15}) for evaluation. We also perform ablation study to prove the necessity of Depthwise Fast Wavelet Transform in our model.

\subsection{CIFAR-10}
We first evaluate our model on the small scale CIFAR-10 dataset. CIFAR-10 contains 50000 training samples and 10000 test samples. We strictly follow the same setting from \cite{He2016DeepRL} to have a fair comparison (i.e. we consider ResNet-$M$). We use Momentum Optimizer with a momentum decay rate $0.9$. The learning rate starts from 0.1 and drops to 0.01 at 82 epochs. The weight decay is set to 0.0002. Batch size is 128. We still use the same architecture from \cite{He2016DeepRL} for our network, i.e. the model contains 3 stages of $M$ identical building blocks. Hence, the total number of layers is $9M + 2$ for our architecture. We use standard data augmentation to have a fair comparison. Here, we choose $M=24$ for WaveletNet.

The test accuracy of our architectures is reported in Table \ref{tbl:cifar10} along with other state-of-the-art efficient models. Our model outperforms MobileNetV2, IGCV2~\cite{IGCV2} and IGCV3~\cite{IGCV3} in terms of the number of parameters.

\begin{table}[h!]
\centering
  \begin{tabular}{l|r|c}
    \toprule
    model        &  \# param &	Error  $(\%)$  \\
    \midrule
    DenseNet-40 \cite{Huang2017DenselyCC} & 1.02M &  6.11 \\
    ResNet-164 \cite{He2016DeepRL}  & 1.70M  & 5.41 \\
    IGCV2\cite{IGCV3}   &   2.20M      & 5.34  \\
    IGCV3\cite{IGCV3}   &   2.20M    & 5.04  \\
    MobileNetV2\cite{Sandler2018MobileNetV2IR}  & 2.10M   & 5.44  \\
    	\midrule
    WaveletNet (ours) 		& 1.10M	& 5.15	\\
    \bottomrule
  \end{tabular}
\caption{Performance of WaveletNet and other state-of-the-art efficient models on CIFAR-10. \label{tbl:cifar10}}  
\end{table}

\subsection{Ablation study}We show that the Depthwise Fast Wavelet Transform is crucial in our architecture by comparing the performances with and without the DFWT layer. Similar to the importance of the Shuffle layer to the ShuffleNet, we see a significant improvement when we use DFWT as a conjugate layer to WConv.
\begin{table}[!ht]
  \centering
  \begin{tabular}{l|c|c|c}
    \hline
    model        &   original
	&  no DFWT
& gap 
 \\
    \hline
    WaveletNet  	& 5.15 & 8.74  & 3.59	\\
    \hline
  \end{tabular}
\caption{Performance of WaveletNet with DFWT on CIFAR-10 vs. the performance of the same architecture without DFWT. Even though there are no parameters in DFWT, the connectivity created by these layers is crucial. Reported numbers are errors on the test set measured in \%.\label{tabl:abl-study}}  
\end{table}

\subsection{ImageNet}
We evaluate our model on the ImageNet 2012 classification dataset. We follow most of the training
settings and hyper-parameters used in \cite{He2016DeepRL}. We also use a similar network architecture as \cite{Sandler2018MobileNetV2IR} with a slight modification. The architecture is listed in Table. \ref{tbl:architecture}.
\begin{table}[h!]
	\small
  \centering
  \begin{tabular}{l|c|c|c|c|c|c|c}
    \hline
    Input	&	Operator & $C$ & repeat & stride & $\kappa_1$ & $\kappa_2$ & $\kappa_3$\\
    \hline
    224$^2$x3 & Conv & 32 & 1 & 2 & - & - & - \\
    112$^2$x32 & block & 16 & 1 & 1 & 0 & 3 & 3 \\
    112$^2$x16 & block & 24 & 2 & 2 & 0 & 3 & 3 \\
    56$^2$x24 & block & 32 & 3 & 2 & 0 & 3 & 3 \\
    28$^2$x32 & block & 64 & 4 & 2 & 0 & 3 & 3 \\
    14$^2$x64 & block & 96 & 3 & 1 & 0 & 3 & 3 \\
    14$^2$x96 & block & 160 & 3 & 2 & 0 & 3 & 3 \\
    7$^2$x160 & block & 320 & 1 & 1 & 0 & 3 & 3 \\
    7$^2$x320 & conv2d & 1280 & 1 & 1 & 0 & 3 & 3 \\
    7$^2$x1280 & avgpool & - & 1 & - & - & - & - \\
    1$^2$x1280 & conv2d & 1000 & - & - & - & - & -  \\
    \hline
  \end{tabular}
  \caption{WaveletNet Architecture for ImageNet. $C$ is the number of outptut channels. $\kappa_1,\kappa_2,\kappa_3$ are the logarithmic depths for our WaveletNet blocks.}\label{tbl:architecture}
\end{table}

We used 4 Tesla V100 GPUs to train our model. The total batch size is set to 512. To benchmark, we compare single crop top-1 performance on the ImageNet validation set, i.e. cropping a 224x224 center view from 256x256 input image and evaluating classification accuracy. 
We strictly follow the same settings as in \cite{He2016DeepRL} for all our models. We used Momentum Optimizer with a starting learning rate 0.1 and decay rate 0.9. We used a linear learning rate decay in total 120 epochs. We tested different sizes of our model and compared to other state-of-the-art CNNs. The results are reported in Table \ref{tbl:imagenet}.

\begin{table}[h!]
  \small
  \centering
  \begin{tabular}{l|r|r|c}
    \toprule
    model      &  \# param	&	\# MAC  &    error $\%$ \\
    \midrule
    AlexNet	\cite{krizhevsky2012imagenet}		&	60M		&	1.5G	&	42.8	\\
	VGG 16 \cite{43022}			&	138M	&	15.5G	&	28.5	\\
    ResNet-50	\cite{He2016DeepRL}		& 	24M	
    &	3.9G	& 	24.0	\\
            \midrule
    NASNet-A~\cite{Zoph2017LearningTA} &	5.3M	
    &	564M	&	26.0	\\
        \midrule
    MobileNet V1~\cite{Howard2017MobileNetsEC}	&	4.2M	&	575M	&	29.4	\\
            \midrule
    MobileNet V2~\cite{Sandler2018MobileNetV2IR}	&	3.4M	&	300M	&	28.0	\\
    MobileNet V2 (x1.4)~\cite{Sandler2018MobileNetV2IR}	&	6.9M	&	585M	&	25.3	\\
        \midrule
	ShuffleNet (x0.5)~\cite{Zhang2017ShuffleNetAE}	&	-	&	38M	& 41.6	\\
	ShuffleNet~\cite{Zhang2017ShuffleNetAE}	&	-	&	140M	&	32.4	\\
	ShuffleNet (x1.5)~\cite{Zhang2017ShuffleNetAE}	&	-	&	292M	&	28.5	\\
    ShuffleNet  (x2)~\cite{Zhang2017ShuffleNetAE} 	&	- 	&	524M	&	26.3	\\
    \midrule
    WaveletNet (ours) 	&  2.72M	&	216M	&  	30.6 \\
    \bottomrule
  \end{tabular}
  \caption{Performance of WaveletNet and other state-of-the-art models on ImageNet.}\label{tbl:imagenet}  
\end{table}

WaveletNet shows comparable results to ShuffleNet and MobileNetV2 in terms of accuracy vs. MAC. It also requires less number of parameters.

\section{Conclusion} \label{sec:concl}
In this paper, we discussed our intuition on building efficient CNN architectures and showed the importance of keeping connectivity within residual blocks. We also addressed a way of learning principles for efficient design from AutoML models and encouraged to explore symmetry-breaking models. We also proposed using human designed fast algorithms in convolutional neural network models to increase efficiency.

Hence, we designed a logarithmic scale convolution layer called Wavelet Convolution, and proposed its conjugate layer called depthwise fast wavelet transformation and then derived an efficient CNN, WaveletNet.
Wavelet Convolutions shrink the complexity of general convolutional layers by $\bigo(\log D / D)$. We compared WaveletNet to a variety of efficient CNN architectures and demonstrated superior size and complexity, and comparable accuracy characteristics.


As the WaveletNet unit mostly operates on the channel dimensions, it can also be considered as a general building block, not necessarily restricted to 2D CNNs. We expect that most 1D or 3D CNNs can also take advantage of our model in applications such as Speech Recognition and Video Classification.

\section*{Acknowledgement}
This material is based in part upon work supported by the Defense 
Advanced
Research Projects Agency (DARPA) under Agreement No. HR00111890042.
This material is also based upon work supported in part by the National 
Science Foundation under
Grant No. CCF-1640012, as well as in part by the Semiconductor Research 
Corporation
under Grant No. 2016-EP-2693-B.
It is also supported in part by the Army Research Office and was 
accomplished under Cooperative Agreement Number
W911NF-18-2-0048, and the MIT-SenseTime Alliance on Artificial 
Intelligence.

{\small
\bibliographystyle{ieee}
\bibliography{citation}
}

\section*{Appendix}
\section*{Formalism on Connectivity}

\subsection{Notations and conventions.} We clarify the following: with the style $\mathsf{C}, \mathsf{C}', \mathsf{C}_1, \mathsf{W}, \mathsf{S},$ etc. we label convolutional layers; every standard 2D convolutional layer (Conv2D) $\mathsf{C}$ has a three dimensional tensor $\mathbf{x}$
of size $N \times N \times D$ as an input, where $D$ is the number of input channels and $N$ is the spatial dimension. The Conv2D layer itself has $D'$ kernels of size $K \times K \times D$. In general, we assume that $D=D'$ unless we specify otherwise. We write input and output tensors of layers as $\mathbf{x}, \mathbf{x}', \mathbf{y},$ etc.; furthermore, in our calculations of Big O complexities, we assume that $N$ and $K$ are constants. There are two types of complexities that we compute: $\bigof$ when counting the MAC and $\bigom$ when counting the number of parameters. When the complexities $\bigof$ and $\bigom$ are the same, we simply write $\bigo$. We label  matrices as $A,B,H,$ etc; let $A = \{a_{ij}\}_{0\leq i,j \leq D-1}$ be a $D \times D$ matrix. With $\ell(A)$ we count the number of nonzero elements of $A$. When all of the elements of $A$ are nonzero, i.e. $a_{ij} \neq 0$, then we say that the matrix $A$ is \textit{full}, which is equivalent to $\ell(A)=D^2$. Finally, we assume that the logirithmic depth $\kappa$ is maximum, i.e. WConv is fully utilized, meanting that it takes $D\log D$ parameters and $N^2D\log D$ MACs (Table~1 in main paper).

\subsection{The notion of connectivity on channel dimensions}

Our first model assumption is to \textit{manipulate channel dimensions} of the inputs, thereby aiming to \textit{reduce the complexity} of the model. To move forward with our assumptions, we introduce an intuitive notion of \textit{connectivity}. 
\begin{definition}[Connectivity] \label{def:conn}
For a convolutional layer $\mathsf{C}$ with $D$ input channels and $D$ output channels we associate an \textit{adjacency matrix} $A(\mathsf{C}) \equiv A \colonequals \{ a_{ij}\}_{0\leq i,j \leq D-1}$, meaning that $a_{ij} = 1$ if there is a sub-convolution $\mathsf{C}'$ in $\mathsf{C}$, where channel $\#i$ is in the input channels of $\mathsf{C}'$ and channel $\#j$ is in the output channels of $\mathsf{C}'$.\footnote{Here the indexation of the channels starts from 0 and goes from left to right.}
\end{definition}
For example, for group convolutions $A$ is a block diagonal matrix with $D/g$ matrices of side $g$ on the diagonal; for depthwise convolutions $A$ is a diagonal (identity) matrix. 

Definition \ref{def:conn} helps us read the complexity of the models through counting the number of nonzero elements in the adjacency matrix $A$. Group convolutions have $g(D/g)^2 = D^2/g$ nonzero elements in $A$, which represents $\mathcal{O}(D^2/g)$ complexity. Likewise, a depthwise convolution is of order $\mathcal{O}(D)$. Conventional convolutions have $\bigo(D^2)$ complexity, which corresponds to $g=1$. It feels natural to ``fill in the gap'' and introduce a construction of the order $\mathcal{O}(D\log D)$ (see equation \eqref{comparison} for a hierarchy). 

\begin{equation} \label{comparison}
\underset{\text{DWConv}}{\bigom(D)} \subset \boxed{\underset{\text{WConv (ours)}}{\bigom(D \log D)}} \subset \underset{\text{GConv}}{\bigom(D^2/g)}
\end{equation}

For WConv the adjacency matrix $A$ will be sparse by virtue of algorithm Algorithm~1 and Figure~3 (in main papaer). More specifically, after direct observations and calculations, the following holds.
\begin{claim}[Structure of the wavelet convolution's adjacency matrix $A$]\label{claim:A-wconv}
For $D=2^k$ and a WConv $\mathsf{C}$ the matrix $A(\mathsf{C}) = \{ a_{ij}\}_{0\leq i,j \leq D-1}$ satisfies $a_{ij} = \mathbb{I}_{(i,j) \in \mathcal{A}}$\footnote{Here $\mathbb{I}_\mathcal{V}$ is the indicator function, meaning that it is 1 if condition $\mathcal{V}$ is true, and it is 0 otherwise.}
and $$\mathcal{A} = (\mathcal{A}' \cup \mathcal{A}'') \cup (\mathcal{A}_1 \cup \cdots \cup \mathcal{A}_{k-1}),$$ where $\mathcal{A}' = \{ (0,j) \mid j \in [0, D/2-1]\}$\footnote{Here $[a,b]$ for $a \leq b$ integers, denotes the discrete interval of integers $a, a+1, \dots b$.}, $\mathcal{A}'' = \{ (j,D-1) \mid j \in [D/2,D-1]\}$ and $\mathcal{A}_p = \{ (i,j) \mid i \in [2^{p-1},2^{p}-1], j \in [D-D/2^p+1,D-D/2^{p+1}]\}$ for all $p=1,\dots,k-1$. In total, $A$ has $D+(D/4)(k-1) = D+(D/4)(\log D - 1)$ nonzero elements (see $A$ in equation \eqref{eq:haar-exmpl} for $D=8$).
\end{claim}
\begin{proof}
The set $\mathcal{A}'$ measures the first convolution of $\mathsf{C}$ that connects the first half of the input channels to the first output channel. Likewise, the set $\mathcal{A}''$ measures the last convolution of $\mathsf{C}$ that connects the last input channel to the second half of output channels. The description of the intermediary sets $\mathcal{A}_p$ comes from simple indexation. Finally, the first and last convolutions have $D/2$ connections. The intermediate convolutions have $D/4$ connections and there are $\log D - 1$ possibilities for $p$. Thus, in total, there are $2(D/2)+(D/4)(\log D -1)$ nonzero elements in $A$, as desired.
\end{proof}

\subsection{Optimal connectivity for WConv}
Our second model assumption is that the adjacency matrix $A$ \textit{needs to be full to facilitate better performance}. For example, the adjacency matrices for group, depthwise and wavelet convolutions are sparse, hence some of the input channels are not connected with the output channels. 

It is reasonable to conjecture that this lack of connectivity might limit the representations of the model, and in fact, we are not the first to notice such an issue. ShuffleNet utilizes a Shuffle layer $\mathsf{S}$ (with no additional parameters) between two group convolutions, $\mathsf{C}_1$ and $\mathsf{C}_2$, which serves to connect isolated channels. 

As the adjacency matrix $A$ can be viewed as a transition matrix, in the language of our formalism, the contribution of the Shuffle layer can be restated as the product $A(\mathsf{C}_2)A(\mathsf{S})A(\mathsf{C}_1)$\footnote{Note that  the adjacency matrix acts on column vectors, thereby acting on the left.} being a full matrix, thereby securing communication between every pair of channels. Since $\mathsf{S}$ is $\bigof(D)$ and $\bigom(1)$ and $\mathsf{C}_i$ is $\bigo(D^2/g)$ for $i=1,2$, achieving a full connectivity has complexity $\bigo(D^2/g)$ and therefore is unlikely to be optimal. 

Here we show that we can achieve full connectivity of complexity $\bigo(D\log D)$ from the sparse structure of WConv. The question is \textit{what kind of an alternative of the Shuffle layer works best for WConv}, i.e. \textit{what is the optimal connectivity} for WConv. Namely, let us formalize the problem, by defining the following two important sets: 
\begin{multline}
\mathcal{H}^{\mathrm{MEM}}_\mathsf{C} \colonequals \{ \text{convolutions } \mathsf{C}' \mid A(\mathsf{C})A(\mathsf{C}')
\text{ is full and } \\ \bigom \text{-complexity of } \mathsf{C}'\text{ is minimal}\}
\end{multline}
\begin{multline}
\mathcal{H}^{\mathrm{MAC}}_\mathsf{C} \colonequals \{ \text{convolutions } \mathsf{C}' \mid A(\mathsf{C})A(\mathsf{C}')
\text{ is full and} \\ \bigof\text{-complexity of }\mathsf{C}'\text{ is minimal}\}
\end{multline}

If we know the convolutional layers in $\mathcal{H}^{\mathrm{MEM}}_\mathsf{C}$ and $\mathcal{H}^{\mathrm{MAC}}_\mathsf{C}$ then we know the answer to our problem. 

The efficient solution to this problem is rooted in wavelet theory. Take $D = 2k$ for simplicity, where $k \geq 1$. The discrete wavelet (Haar) transform is a change of basis, transforming a $D$-dimensional vector from the standard basis to the Haar basis. Therefore, the transform can be described by an $D \times D$ matrix. Here, with the
matrix $H_D$ we define our modified wavelet, tweaked from the Haar wavelet, as follows:
\begin{equation}
H2 :=
\begin{bmatrix}
1 & −1 \\
1 & 1 
\end{bmatrix}
\end{equation}
and for $k > 1$ the recursive definition 
\begin{equation}
\label{eq:haar-intro}
H_D \equiv H_2^k :=
\begin{bmatrix}
I_2^{k-1} \otimes [1, −1] \\
H_2^{k−1} \otimes [1, 1] 
\end{bmatrix},
\end{equation}
where $I_2^{k-1}$ is the $2^{k-1} \times 2^{k-1}$ identity matrix and $\otimes$ is the Kronecker product. Here, we only relabeled the rows and columns of the original Haar matrix. The matrix $H_D$ has orthogonal rows and columns. To perform the DFWT layer we broadcast the matrix $H_D$ on channel dimensions.

Now, we can study the connectivity of channel dimensions once we have DFWT and WConv connected in sequence. Namely, let $\mathsf{W}$ be the DFWT and $\mathsf{C}$ be the WConv layers (see \eqref{eq:haar-exmpl} for $H_D$ and $A=A(\mathsf{C})$). It is also straightforward to see that $A(\mathsf{W})=|H_D|$, i.e. we take the absolute value element-wise, and thus convert $-1$ entries to $1$. See the following example for $D=8$. 

\begin{eqnarray} 
\centering
\label{eq:haar-exmpl}
\small
A&:&
\begin{bmatrix}
    \color{red} \mathbf{1} & \color{red} \mathbf{1} & \color{red} \mathbf{1} & \color{red} \mathbf{1} & \color{red} 0 & \color{red} 0 & \color{red} 0 & \color{red} 0\\
  \color{green} 0 & \color{green} 0 & \color{green} 0 & \color{green} 0 & \color{green} \mathbf{1} & \color{green} \mathbf{1} &  \color{green} 0 & \color{green} 0 \\
    0 & 0 & 0 & 0 & 0 & 0  & \mathbf{1} & 0 \\
    0 & 0 & 0 & 0 & 0 & 0 & \mathbf{1} & 0 \\
    \color{blue} 0 & \color{blue} 0 & \color{blue} 0 & \color{blue} 0 & \color{blue} 0 & \color{blue} 0 & \color{blue} 0 & \color{blue} \mathbf{1} \\
    \color{blue} 0 & \color{blue} 0 & \color{blue} 0 & \color{blue} 0 & \color{blue} 0 & \color{blue} 0 & \color{blue} 0 & \color{blue} \mathbf{1} \\
    \color{blue} 0 & \color{blue} 0 & \color{blue} 0 & \color{blue} 0 & \color{blue} 0 & \color{blue} 0 & \color{blue} 0 & \color{blue} \mathbf{1} \\
    \color{blue} 0 & \color{blue} 0 & \color{blue} 0 & \color{blue} 0  & \color{blue} 0 & \color{blue} 0 & \color{blue} 0 & \color{blue} \mathbf{1}
\end{bmatrix} \\
H_8&:&
\small
\begin{bmatrix}
  \color{red} \mathbf{1} & \color{red} 0 & \color{red} 0 & \color{red} 0 & \color{red} \mathbf{-1} & \color{red} 0 & \color{red} 0 & \color{red} 0\\
  \color{red} 0 & \color{red} \mathbf{1} & \color{red} 0 & \color{red} 0 & \color{red} 0 & \color{red} \mathbf{-1} & \color{red} 0 & \color{red} 0 \\
    \color{red} 0 & \color{red} 0 & \color{red} \mathbf{1} & \color{red} 0 & \color{red} 0 & \color{red} 0  & \color{red} \mathbf{-1} & \color{red} 0 \\
    \color{red} 0 & \color{red} 0 & \color{red} 0 & \color{red} \mathbf{1} & \color{red} 0 & \color{red} 0 & \color{red} 0 & \color{red} \mathbf{-1} \\
    \color{green} \mathbf{1} & \color{green} 0 & \color{green} \mathbf{-1} & \color{green} 0 & \color{green} \mathbf{1} & \color{green} 0 & \color{green} \mathbf{-1} & \color{green} 0 \\
    \color{green} 0 & \color{green} \mathbf{1} & \color{green} 0 & \color{green} \mathbf{-1} & \color{green} 0 & \color{green} \mathbf{1} & \color{green} 0 & \color{green} \mathbf{-1} \\
    \mathbf{1} & \mathbf{-1} & \mathbf{1} & \mathbf{-1} & \mathbf{1} & \mathbf{-1} & \mathbf{1} & \mathbf{-1} \\
    \color{blue} \mathbf{1} & \color{blue} \mathbf{1} & \color{blue} \mathbf{1} & \color{blue} \mathbf{1}  & \color{blue} \mathbf{1} & \color{blue} \mathbf{1} & \color{blue} \mathbf{1} & \color{blue} \mathbf{1}
\end{bmatrix} 
\normalsize
\end{eqnarray}

Moreover, a simple matrix multiplication can show us that $A(\mathsf{C})A(\mathsf{W})$ \textit{is} a full matrix. What remains to be seen is that the DFWT $\mathsf{W}$ is optimal under the assumptions of our connectivity formalism.

\begin{theorem}[DFWT is Optimal for WConv] \label{thm:opt}
Let $D=2^k$. For a WConv $\mathsf{C}$ with $D$ input channels, the DFWT $\mathsf{W}$ with $D$ input channels satisfies 
$$
\mathsf{W} \in \mathcal{H}^{\mathrm{MEM}}_\mathsf{C} \cap \mathcal{H}^{\mathrm{MAC}}_\mathsf{C}.
$$
\end{theorem}
\begin{proof} Let $\tilde{A} = A(\mathsf{C})$. It suffices to show that $\mathsf{W}$ is in both $\mathcal{H}^{\mathrm{MEM}}_\mathsf{C}$ and $\mathcal{H}^{\mathrm{MAC}}_\mathsf{C}$. Firstly, since $A(\mathsf{W})=|H_D|$, by the recursive definition of $H_D$ in equation \eqref{eq:haar-intro} we can compute that $\tilde{A}(A(\mathsf{W}))=\tilde{A}|H_D|$ is full (see the $D=8$ case in equation \eqref{eq:haar-exmpl}). Secondly, suppose that $\tilde{A}D$ is full for some $D \times D$ matrix $D$. Let $a_0,\dots a_{D-1}$ be the row vectors of $\tilde{A}$ and $d_0,\dots d_{D-1}$ be the column vectors of $D$. Hence, all of the possible dot products $a_id_j$ are nonzero. Claim \ref{claim:A-wconv} yields a classification of the nonzero entries in $\tilde{A}$, given by the sets $\mathcal{A}', \mathcal{A}'', \mathcal{A}_p$. We can color the rows of $A$ depending on which set contains the rows' nonzero indices (see $A$ in equation \eqref{eq:haar-exmpl}). Since a row of a specified color is sparse, the nonzero entries of that row determine a rectangular in $D$ (labeled with the same corresponding color) where there should be at least a single $1$ in the columns of the region (see $H_8$ in equation \eqref{eq:haar-exmpl}), otherwise some dot product $a_id_j$ will be zero. Enforcing minimality of $\bigom$ for $D$, it means that $\ell(D)$ has to be minimal, i.e. there is \textit{exactly} one nonzero element in each column of each separate color region in $D$. By virtue of its definition, $H_D$ satisfies the condition of minimality. Therefore, $|H_D|=A(\mathsf{W})$ is $\bigom$-minimal among all $D$, i.e. $\mathsf{W} \in \mathcal{H}^{\mathrm{MAC}}_\mathsf{C},$ as desired. Finally, since any layer $\mathsf{C}'$ receives input from $D$ channels, which in principle can be arbitrary, then the number of MACs of that layer has a linear lower bound. As the DFWT is $\bigof(D)$, this lower bound is sharp, hence $\mathsf{W} \in \mathcal{H}^{\mathrm{MAC}}_\mathsf{C},$ as desired.
\end{proof}

\section*{Additional Experimental Details}

We tested WaveletNet with different channel sizes and different $\kappa$ values on ImageNet. 

\begin{table}[h!]
  \small
  \centering
  \begin{tabular}{l|r|r|c}
    \toprule
    model      &  \# param	&	\# MAC  &    error $\%$ \\
    \midrule
    WaveletNett (x0.75)	&  1.79M	&	140M	&  	32.7  \\
    WaveletNet (x1) 	&  2.72M	&	216M	&  	30.6 \\
    WaveletNet (x1.25) 	&  3.82M	&	327M	&  	28.5 \\
    \bottomrule
  \end{tabular}
  \caption{Performance of WaveletNet with different channel on ImageNet.}\label{tbl:imagenet1}  
\end{table}

\begin{table}[h!]
  \small
  \centering
  \begin{tabular}{l|r|r|c}
    \toprule
    model      &  \# param	&	\# MAC  &    error $\%$ \\
    \midrule
    WaveletNet ($\kappa=5$) 	&  2.60M	&	204M	&  	31.5 \\
    WaveletNet ($\kappa=3$) 	&  2.72M	&	216M	&  	30.6 \\
    \bottomrule
  \end{tabular}
  \caption{Performance of WaveletNet with different $\kappa$ on ImageNet.}\label{tbl:imagenet2}  
\end{table}

Here $\kappa$ refers to the depth of the DFWT and the second WConv layer in Figure~5 (main paper) and we still set the first WConv as fully connected i.e. $\kappa_1=0; \kappa_2=\kappa_3=\kappa.$

Tables~\ref{tbl:imagenet1} and~\ref{tbl:imagenet2} show that WaveletNet is comparable to state-of-the-art efficient CNN models.

\end{document}